\documentclass{article}
\usepackage{iclr2026_conference,times}

\usepackage{hyperref}
\usepackage{url}
\usepackage{booktabs}
\usepackage{amsfonts}
\usepackage{amsmath}
\usepackage{amssymb}
\usepackage{amsthm}
\newtheorem{proposition}{Proposition}
\usepackage{bbm}  
\DeclareMathOperator*{\argmax}{arg\,max}
\usepackage{nicefrac}
\usepackage{microtype}
\usepackage{graphicx}
\usepackage{xcolor}
\usepackage{tikz}
\usetikzlibrary{positioning, arrows.meta, shapes.geometric, fit}
\usepackage{algorithm}
\usepackage{algorithmic}
\usepackage{enumitem}
\usepackage[most]{tcolorbox}

\title{Towards Expert Financial QA via Self-Improving RAG}

\author{
  Junjie Xiong\thanks{Equal contribution.} \\
  University of California, Berkeley \\
  \And
  Shawheen Ghezavat\footnotemark[1] \\
  California Polytechnic State University \\
  \And
  Aum Hirpara\footnotemark[1] \\
  Hofstra University
}

\iclrfinalcopy

\begin{document}
\raggedbottom  

\maketitle

\begin{abstract}
Expert-level financial question answering requires both \textbf{grounded verification} to catch numeric hallucinations and \textbf{audit trails} for regulatory compliance, attributes that standard single-pass RAG systems lack. We take a step toward this goal with Self-Improving RAG, a framework that decomposes document QA into three specialized agents (Retrieval, Reasoning, and Judge) coordinated by an orchestrator with feedback-driven self-correction. When the Judge Agent scores an answer below a dynamic threshold, the system triggers retry with escalated strategies: broader retrieval, more careful prompting, and relaxed acceptance criteria. We evaluate on FinanceBench (SEC filing QA), where Self-Improving RAG achieves 86\% oracle-guided accuracy (measuring agreement with gold answers) with a 36.4\% Lazarus Rate, recovering nearly 4 in 10 initially incorrect answers through targeted retry. A key finding is that a fixed retrieval pipeline with judge-driven retry achieves strong results without dynamic routing, providing full interpretability. Every decision is logged with confidence scores, enabling the audit trails required for regulated financial applications.
\end{abstract}


\section{Introduction}
\label{sec:intro}

Retrieval-augmented generation (RAG) has become the standard paradigm for knowledge-intensive question answering~\citep{lewis2021retrievalaugmentedgenerationknowledgeintensivenlp}. However, as RAG systems are deployed in high-stakes domains such as finance~\citep{islam2023financebenchnewbenchmarkfinancial, tai2025veritasfiadaptablemultitieredrag}, a critical limitation emerges: conventional single-pass pipelines lack the ability to recognize and correct their own failures~\citep{asai2023selfraglearningretrievegenerate, yan2024correctiveretrievalaugmentedgeneration}.

Consider a financial analyst querying SEC filings to extract quarterly revenue figures. A single-pass RAG system retrieves documents, generates an answer, and returns with no mechanism to verify correctness. For financial professionals, an incorrect answer is worse than no answer, and a system without an \textbf{audit trail} is indistinguishable from a guess.

\paragraph{The Walled Garden Constraint.}
Financial applications impose a crucial constraint that distinguishes them from general-domain QA: retrieval must remain within authorized document corpora. Unlike systems that can fall back to web search when initial retrieval fails~\citep{yan2024correctiveretrievalaugmentedgeneration}, financial QA systems operate in a ``walled garden'' where data governance policies prohibit external information sources. This constraint eliminates a common recovery mechanism and demands alternative approaches to self-correction. Importantly, this closed-domain constraint also serves as an \emph{agent governance mechanism}: by limiting retrieval scope, we ensure every answer traces to authorized, auditable sources, a key requirement for responsible deployment of agentic systems in regulated industries.

\begin{figure}[t!]
\centering
\includegraphics[width=\linewidth]{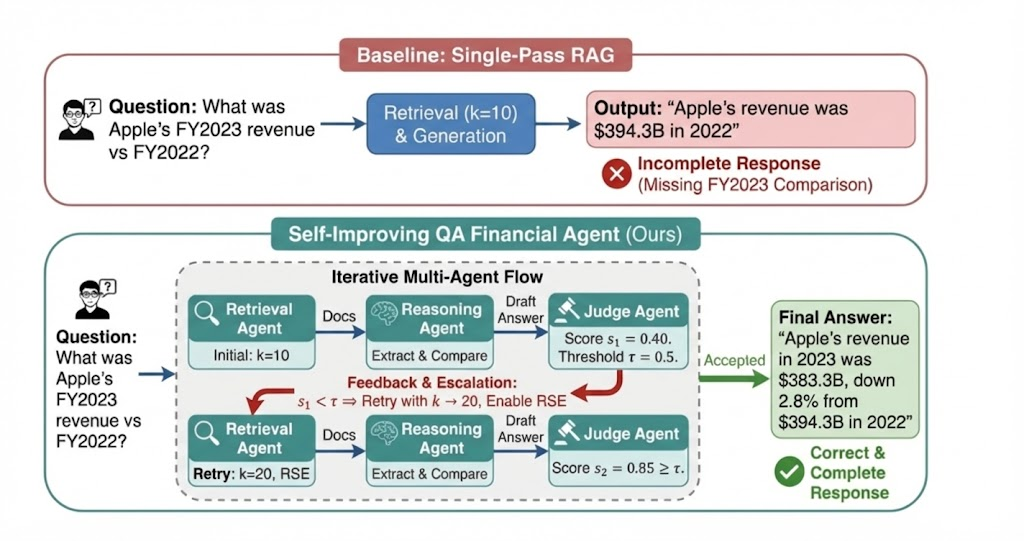}
\caption{Three-agent architecture: Retrieval, Reasoning, and Judge agents coordinated by an Orchestrator. When score $s_1 < \tau$, our system retries with escalated retrieval ($k{:}10{\rightarrow}20{\rightarrow}30$) and adapted prompting until accepted or budget exhausted.}
\label{fig:architecture}
\end{figure}

In finance specifically, three additional challenges compound: (1)~\emph{numeric reasoning} over tables and financial statements, (2)~\emph{temporal filtering} requiring understanding of fiscal years and reporting periods, and (3)~\emph{entity disambiguation} among ticker symbols, subsidiaries, and corporate name changes.

\paragraph{Our Approach.}
We propose \textbf{Self-Improving RAG}, a framework that decomposes retrieval-augmented generation into three specialized agents coordinated by an orchestrator with feedback-driven self-correction. The \textbf{Retrieval Agent} analyzes queries, extracts entities, and selects from available retrieval pipelines with escalation based on attempt number. The \textbf{Reasoning Agent} generates answers with explicit citations and adapts its prompting strategy based on prior attempt outcomes. The \textbf{Judge Agent} evaluates answer quality against dynamic thresholds and decides whether to accept or request retry.

The key insight is that when initial retrieval or generation fails, the system can \emph{escalate} rather than return low-quality answers: retrieve more documents, prompt more carefully, and accept marginally lower confidence when effort has been expended. This contrasts with prior adaptive retrieval approaches that focus on initial routing without recovery mechanisms~\citep{jeong2024adaptiveraglearningadaptretrievalaugmented, asai2023selfraglearningretrievegenerate}.

Our contributions toward expert-level financial QA are:
\begin{enumerate}[nosep]
    \item \textbf{Structured Self-Correction for Document QA}: Unlike prior single-model self-reflection~\citep{madaan2023selfrefineiterativerefinementselffeedback}, we introduce \emph{targeted escalation} across retrieval, generation, and evaluation stages, with specialized agents diagnosing whether failure originated in retrieval or reasoning
    \item \textbf{Grounded Verification}: A Judge that performs entailment checking against retrieved evidence, with programmatic numeric verification for financial figures
    \item \textbf{Audit-First Design}: Every agent decision logged with provenance, confidence scores, and reasoning traces, enabling compliance review for regulated industries
\end{enumerate}

On FinanceBench, Self-Improving RAG achieves 86\% LLM Judge accuracy with a 36.4\% Lazarus Rate, recovering nearly 4 in 10 initially incorrect answers. A key finding is that a fixed retrieval pipeline with judge-driven retry achieves strong results without dynamic routing, providing full interpretability.


\section{Related Work}
\label{sec:related}

\paragraph{Self-Correction in Language Models.}
The emerging paradigm of Agentic RAG~\citep{singh2025agenticretrievalaugmentedgenerationsurvey} embeds autonomous agents into retrieval pipelines to enable self-correction and adaptive reasoning beyond single-pass systems.
Reflexion~\citep{shinn2023reflexionlanguageagentsverbal} uses verbal reinforcement learning across episodes; Self-Refine~\citep{madaan2023selfrefineiterativerefinementselffeedback} achieves 15--20\% gains via single-model critique-refine loops. Self-RAG~\citep{asai2023selfraglearningretrievegenerate} trains models to emit reflection tokens, while CRAG~\citep{yan2024correctiveretrievalaugmentedgeneration} triggers corrective retrieval on failure. Unlike Reflexion, we correct \emph{within a single session}; unlike Self-Refine, we use \emph{specialized agents} with domain-specific verification; unlike Self-RAG, we require no fine-tuning. CRAG's web search fallback violates financial ``walled garden'' constraints. Self-correction can also be learned via RL~\citep{kumar2024traininglanguagemodelsselfcorrect}, which defines a \emph{correction rate} metric conceptually similar to our Lazarus Rate; our approach achieves self-correction without RL training, using heuristic escalation and specialized judging. Crucially, our Judge Agent diagnoses whether failure originated in retrieval \emph{or} reasoning and triggers targeted retry, unifying both correction modalities.

\paragraph{Adaptive Retrieval and Document QA.}
Adaptive-RAG~\citep{jeong2024adaptiveraglearningadaptretrievalaugmented} routes queries to different pipelines based on complexity. Unlike approaches focusing on \emph{initial} selection, we combine routing with \emph{judge-driven retry}: if the first attempt fails, the system escalates rather than returning low-quality answers. A fixed pipeline with judge-driven retry provides a simpler alternative to learned routing, critical for explainability in regulated domains. LLM Judge~\citep{zheng2023judgingllmasajudgemtbenchchatbot} provides scalable quality assessment; recent surveys~\citep{li2024llmsasjudgescomprehensivesurveyllmbased} highlight judge biases that motivate our separate Judge Agent design. Agent-as-a-Judge~\citep{zhuge2024agentasajudgeevaluateagentsagents} achieves 90\% human agreement via multi-turn evaluation, supporting our iterative feedback approach.

\paragraph{Multi-Agent Coordination.}
AutoGen~\citep{wu2023autogenenablingnextgenllm} and MetaGPT~\citep{hong2024metagptmetaprogrammingmultiagent} pioneered agentic frameworks for multi-turn LLM coordination. Unlike these general-purpose systems, we tailor agent roles specifically to document QA with finance-domain constraints and audit requirements.

\paragraph{Positioning: The Walled Garden Constraint.}
A key distinction of our work is the \emph{closed-domain} constraint: retrieval must remain within authorized documents, precluding web search fallbacks that violate financial data governance policies. Table~\ref{tab:method-comparison} summarizes how Self-Improving RAG relates to prior methods across four dimensions critical for financial applications.

\begin{table}[ht]
\centering
\small
\newcommand{\cmark}{\textcolor{green!70!black}{\checkmark}}%
\newcommand{\xmark}{\textcolor{red}{\texttimes}}%
\begin{tabular}{lcccc}
\toprule
\textbf{Method} & \textbf{Self-Correct} & \textbf{No Fine-tune} & \textbf{Closed-Domain} & \textbf{Audit Trail} \\
\midrule
Self-RAG~\citep{asai2023selfraglearningretrievegenerate} & \cmark & \xmark & \cmark & \xmark \\
CRAG~\citep{yan2024correctiveretrievalaugmentedgeneration} & \cmark & \cmark & \xmark~ & \xmark \\
Adaptive-RAG~\citep{jeong2024adaptiveraglearningadaptretrievalaugmented} & \xmark & \cmark & \cmark & \xmark \\
Reflexion~\citep{shinn2023reflexionlanguageagentsverbal} & \cmark & \cmark & \xmark & \xmark \\
\midrule
\textbf{Self-Improving RAG (Ours)} & \cmark & \cmark & \cmark & \cmark \\
\bottomrule
\end{tabular}
\caption{Comparison with related self-correction and adaptive retrieval methods. Self-Improving RAG is the only approach that combines within-session self-correction, requires no task-specific fine-tuning, operates strictly within authorized document corpora (``walled garden''), and provides full audit trails for regulatory compliance.}
\label{tab:method-comparison}
\end{table}

\textbf{Self-RAG} trains models to emit special reflection tokens that trigger self-correction, but requires fine-tuning on curated (input, output, reflection) triples, which limits domain adaptation. \textbf{CRAG} uses web search as a fallback when initial retrieval fails, which violates data governance requirements in regulated industries. \textbf{Adaptive-RAG} learns to route queries to different pipelines but lacks a retry mechanism for recovery. Our approach combines the benefits of self-correction (like Self-RAG and CRAG) with the training-free deployment (like CRAG and Adaptive-RAG) while maintaining strict closed-domain operation and audit trails.


\section{Method: Self-Improving RAG}
\label{sec:method}

\paragraph{Problem Setting.}
Given a natural language query $q \in \mathcal{Q}$ and an authorized corpus $\mathcal{D} = \{d_1, \ldots, d_n\}$ of financial documents, produce an answer $a$ supported by evidence $E \subseteq \mathcal{D}$. We assume a \emph{closed-domain} setting where retrieval must remain within authorized documents, a regulatory constraint precluding web search fallbacks. The system may attempt up to $B+1$ total attempts (where $B$ is the retry budget), accepting when utility exceeds threshold or returning the best answer if budget exhausts.

\subsection{Preliminaries and Notation}
\label{sec:method_prelim}

The system state at attempt $t$ is:
\begin{equation}
\mathcal{S}_t = \langle E_t, a_t, \mathbf{s}_t, \tau_t \rangle
\label{eq:state}
\end{equation}
where $E_t \subseteq \mathcal{D}$ is retrieved evidence, $a_t$ is the candidate answer, $\mathbf{s}_t \in [0,1]^3$ is the quality vector, and $\tau_t$ is the acceptance threshold. Three specialized agents, Retrieval ($\mathcal{R}$), Reasoning ($\mathcal{G}$), and Judge ($\mathcal{J}$), are coordinated by an Orchestrator (Figure~\ref{fig:architecture}).

\paragraph{Retrieval Agent.}
The retrieval operator maps query and corpus to evidence:
\begin{equation}
E_t = \mathcal{R}(q, \mathcal{D}; \boldsymbol{\phi}_t)
\label{eq:retrieval}
\end{equation}
where $\boldsymbol{\phi}_t = \{k_t, \sigma_t\}$ parameterizes top-$k$ retrieval and RSE expansion. We use a hybrid pipeline (dense + BM25 + reranking). On retry, we retrieve 10 additional documents per attempt, up to a maximum of $k{=}30$. Relevant Segment Extraction (RSE) activates only on the final attempt. Table~\ref{tab:retrieval-escalation} shows the escalation configuration.

\begin{table}[ht]
\centering
\small
\begin{tabular}{lccc}
\toprule
\textbf{Attempt} & \textbf{top\_k} & \textbf{initial\_k} & \textbf{RSE} \\
\midrule
1 (Standard) & 10 & $3 \times$ & Off \\
2 (Escalated) & 20 & $4 \times$ & Off \\
3 (Maximum) & 30 & $6 \times$ & On \\
\bottomrule
\end{tabular}
\caption{Retrieval Agent escalation strategies. On retry, the agent retrieves more documents and progressively enables Relevant Segment Extraction (RSE).}
\label{tab:retrieval-escalation}
\end{table}

\paragraph{Routing Heuristics.}
The rule-based router selects retrieval pipelines based on query characteristics:
\begin{itemize}[nosep]
    \item If the question contains a recognized ticker symbol or company name $\rightarrow$ hybrid retrieval with metadata filtering
    \item If the question requests numerical comparison or computation $\rightarrow$ hybrid retrieval with metadata filtering and reranking (precision-focused)
    \item If the question is open-ended or exploratory $\rightarrow$ hybrid retrieval with broad recall
    \item Default fallback $\rightarrow$ semantic-only retrieval
\end{itemize}
Entity recognition uses a simple gazetteer of S\&P 500 tickers plus regex patterns for fiscal year mentions. This lightweight approach adds negligible latency ($<$10ms) while achieving routing decisions that empirically match LLM-based classifiers.

\paragraph{Finance Lexicon \& Normalization.}
Financial queries frequently use shorthand and domain jargon (e.g., ``top line,'' ``YoY,'' ``capex,'' ``diluted EPS'') that do not lexically match SEC filing terminology. We therefore maintain a lightweight finance lexicon consisting of (i)~\emph{entity aliases and identifiers} (tickers, company names, subsidiaries), (ii)~\emph{metric canonicalizations} (synonym sets mapping ``top line'' $\to$ revenue, ``SG\&A'' $\to$ operating expenses), and (iii)~\emph{unit/period normalizers} (thousands/millions/billions; fiscal-year and quarter expressions). The Retrieval Agent uses the lexicon for query normalization and synonym expansion prior to hybrid retrieval, while the Judge performs strict numeric comparison (see Section~\ref{sec:results} for discussion of unit normalization opportunities). See Appendix~\ref{app:lexicon} for sample entries.

\paragraph{Reasoning Agent.}
Given query and evidence, the generator samples:
\begin{equation}
a_t \sim P_\theta(a \mid q, E_t, \pi_t)
\label{eq:reasoning}
\end{equation}
where $\theta$ denotes LLM parameters and $\pi_t$ ranges over prompting regimes (\textsc{Standard} $\rightarrow$ \textsc{Conservative} $\rightarrow$ \textsc{Detailed}) that escalate with $t$. Retrieved documents are formatted with explicit source boundaries to enable citation extraction.

\paragraph{Judge Agent.}
The judge maps $(q, a_t, E_t)$ to a quality vector:
\begin{equation}
\mathbf{s}_t = \mathcal{J}(q, a_t, E_t) = [\mu_g, \mu_c, \mu_n]^\top \in [0,1]^3
\label{eq:judge_vector}
\end{equation}
separating \emph{grounding} $\mu_g$ (evidence entailment), \emph{completeness} $\mu_c$ (query coverage), and \emph{numeric faithfulness} $\mu_n$. We aggregate into scalar utility:
\begin{equation}
U_t = \mathbf{w}^\top \mathbf{s}_t = w_g \mu_g + w_c \mu_c + w_n \mu_n, \quad \mathbf{w} \in \mathbb{R}_{\geq 0}^3, \; \mathbf{1}^\top\mathbf{w} = 1
\label{eq:utility}
\end{equation}
where $\mu_g \equiv J_{\text{ground}}, \mu_c \equiv J_{\text{complete}}, \mu_n \equiv J_{\text{numeric}}$ (formalized in Appendix~\ref{app:formal}).

\textbf{Finance constraint.} We weight numeric faithfulness heavily ($w_n = 0.5$, $w_g = 0.3$, $w_c = 0.2$), penalizing numeric errors even when answers appear fluent. Numeric faithfulness uses strict set coverage:
\begin{equation}
\mu_n(a_t, E_t) = \mathbbm{1}\bigl(|N(a_t) \setminus N(E_t)| = 0\bigr)
\label{eq:numeric}
\end{equation}
where $N(\cdot)$ extracts normalized numeric quantities. This catches \emph{near-miss} hallucinations by requiring every number in $a_t$ be explicitly supported by $E_t$.

\paragraph{Orchestrator.}
The orchestrator accepts or retries based on:
\begin{equation}
\tau_t = \max(\tau_0 - \lambda(t-1), \tau_{\text{min}}), \quad \mathbb{D}_t = \mathbbm{1}(U_t \geq \tau_t)
\label{eq:threshold}
\end{equation}
with $\tau_0{=}0.5$, $\lambda{=}0.1$, $\tau_{\text{min}}{=}0.3$. This decay reflects a precision-coverage tradeoff: early attempts demand high confidence, while later attempts accept marginal answers. If no attempt is accepted, return the best:
\begin{equation}
a^\star = a_{\hat{t}}, \quad \hat{t} = \arg\max_{t \in \{1,\ldots,T\}} U_t.
\label{eq:best}
\end{equation}
All decisions are logged with timestamps and reasoning traces for audit compliance.

\begin{proposition}[Convergence Rate]
\label{prop:convergence}
The probability of system failure after $T$ attempts decays multiplicatively, bounded by the conditional failure probability of each stage:
\begin{equation}
\mathcal{P}_{\text{fail}} = \prod_{t=1}^{T} P(\mathbb{D}_t = 0 \mid \mathcal{S}_{t-1}).
\label{eq:failure_prob}
\end{equation}
See Appendix~\ref{app:formal} for the formal proof.
\end{proposition}

\subsection{Orchestrator Algorithm}
\label{sec:orchestrator}

Algorithm~\ref{alg:orchestrator} formalizes the orchestration loop. The key insight is that the system maintains the \emph{best answer seen so far}, ensuring that retry never degrades output quality. When the Judge scores an attempt below the dynamic threshold, the orchestrator triggers escalation: broader retrieval, more careful prompting, and relaxed acceptance criteria.

\begin{algorithm}[ht]
\caption{Self-Improving RAG Orchestrator}
\label{alg:orchestrator}
\begin{algorithmic}[1]
\REQUIRE Question $q$, retry budget $B$
\STATE $\text{best\_answer} \gets \emptyset$, $\text{best\_score} \gets 0$
\STATE $\text{attempt} \gets 1$
\WHILE{$\text{attempt} \leq B + 1$}
    \STATE $\text{docs} \gets \text{RetrievalAgent.retrieve}(q, \text{attempt})$
    \STATE $\text{answer} \gets \text{ReasoningAgent.generate}(q, \text{docs}, \text{attempt})$
    \STATE $\text{score} \gets \text{JudgeAgent.evaluate}(q, \text{answer})$
    \IF{$\text{score} > \text{best\_score}$}
        \STATE $\text{best\_answer} \gets \text{answer}$
        \STATE $\text{best\_score} \gets \text{score}$
    \ENDIF
    \IF{$\neg \text{JudgeAgent.should\_retry}(\text{score}, \text{attempt})$}
        \STATE \textbf{break}
    \ENDIF
    \STATE $\text{attempt} \gets \text{attempt} + 1$
\ENDWHILE
\RETURN $\text{best\_answer}$, $\text{best\_score}$
\end{algorithmic}
\end{algorithm}

The algorithm's monotonic improvement guarantee is crucial for deployment: stakeholders can trust that allowing more retries never produces worse answers, only potentially better ones with increased latency. Table~\ref{tab:notation_full} in Appendix~\ref{app:formal} summarizes notation.


\section{Experiments and Results}
\label{sec:results}

\subsection{Experimental Setup}
\label{sec:setup}

We evaluate on \textbf{FinanceBench}~\citep{islam2023financebenchnewbenchmarkfinancial}, a benchmark of 150 SEC filing questions where 66\% require numerical calculations. We set retry budget $B{=}2$ and initial threshold $\tau_0{=}0.5$. Implementation uses GPT-4o-mini for generation, BGE-large embeddings~\citep{chen2025m3embeddingmultilingualitymultifunctionalitymultigranularity} with ChromaDB, and cross-encoder reranking. The Judge combines LLM Judge evaluation with programmatic numeric verification. Full details in Appendix~\ref{app:impl}.

\paragraph{Evaluation Protocol: Oracle-Guided vs.\ Deployment Modes.}
\textbf{Critical limitation:} We report results under two evaluation protocols that reveal a significant gap. In \textbf{Deployment mode} (Table~\ref{tab:ablation-main}), where the Judge operates blind without gold answers, we achieve only \textbf{31\%} acceptance rate. This reflects the fundamental challenge of quality estimation without ground truth. In \textbf{Oracle-Guided mode} (Table~\ref{tab:financebench-main}), where the Judge has gold-answer access for direct comparison with prior work, we achieve \textbf{86\%} accuracy, demonstrating the system's potential when judge quality improves. The 36.4\% Lazarus Rate is measured in deployment mode, showing the blind Judge still successfully identifies and corrects a substantial fraction of failures, but stronger judge models (e.g., Claude Opus, GPT-4o) could close this gap.

\subsection{Main Results}

Table~\ref{tab:financebench-main} compares single-pass RAG against Self-Improving RAG on FinanceBench using oracle-guided evaluation.

\begin{table}[ht]
\centering
\small
\begin{tabular}{lcc}
\toprule
\textbf{Configuration} & \textbf{Correctness} & \textbf{$\Delta$} \\
\midrule
Single-pass RAG & 53\% {\scriptsize[45, 61]} & -- \\
Self-Improving RAG & \textbf{86\%} {\scriptsize[80, 91]} & \textbf{+62.3\%} \\
\bottomrule
\end{tabular}
\caption{FinanceBench correctness with 95\% bootstrap confidence intervals (oracle-guided evaluation). Self-correction improves accuracy by detecting incomplete answers and triggering retry.}
\label{tab:financebench-main}
\end{table}

\paragraph{Performance by Question Type.}
Table~\ref{tab:financebench-breakdown} breaks down results by FinanceBench question category, revealing that self-correction provides the largest gains on domain-relevant questions (+81.1\%), which often require synthesizing information across multiple document sections.

\begin{table}[ht]
\centering
\small
\begin{tabular}{lccc}
\toprule
\textbf{Question Type} & \textbf{Single-Pass} & \textbf{Self-Corr.} & \textbf{$\Delta$} \\
\midrule
Metrics-generated (66\%) & 53\% & 82\% & +54.7\% \\
Domain-relevant (22\%) & 53\% & 96\% & \textbf{+81.1\%} \\
Novel-generated (12\%) & 53\% & 80\% & +50.9\% \\
\midrule
\textit{Overall} & \textit{53\%} & \textit{86\%} & \textit{+62.3\%} \\
\bottomrule
\end{tabular}
\caption{Correctness by question type. Domain-relevant questions benefit most from self-correction (+81.1\%), as these often have incomplete first-pass answers.}
\label{tab:financebench-breakdown}
\end{table}

\paragraph{Lazarus Rate: Measuring Resilience.}
We measure the \textbf{Lazarus Rate} to quantify self-correction effectiveness: the percentage of initially incorrect answers successfully corrected through retry. Let $\mathcal{W}_1 = \{q : J_1(q) < \tau_{\text{correct}}\}$ denote initially incorrect answers. The Lazarus Rate measures recovery:
\begin{equation}
    \text{LazarusRate} = \frac{|\{q \in \mathcal{W}_1 : J_2(q) \geq \tau_{\text{correct}}\}|}{|\mathcal{W}_1|} = P(\text{correct}_2 \mid \text{wrong}_1).
\end{equation}
On FinanceBench, the Lazarus Rate is \textbf{36.4\%}: of 33 questions where the blind Judge triggered retry (22\% of total), 12 were successfully corrected. Appendix~\ref{app:results} provides detailed correction flow analysis; Appendix~\ref{app:cases} presents case studies and error taxonomy.

\paragraph{Correction Flow Visualization.}
Figure~\ref{fig:correction-sankey} visualizes the complete ``life of a question'' through our self-correction pipeline, showing how 150 questions flow through the system.

\begin{figure}[ht]
\centering
\includegraphics[width=\linewidth]{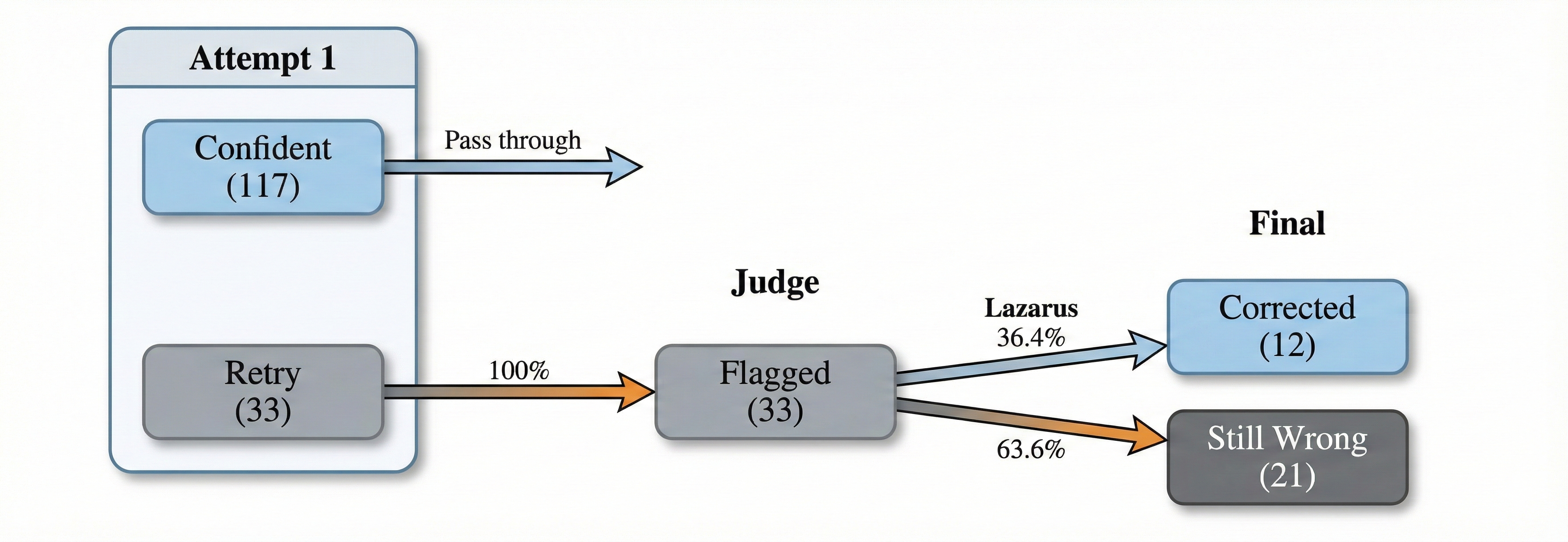}
\caption{Correction flow visualization. The \textbf{Lazarus Rate} represents the proportion of initially incorrect answers successfully corrected through retry.}
\label{fig:correction-sankey}
\end{figure}

\paragraph{Component Ablation.}
Table~\ref{tab:ablation-main} reports ablations in deployment mode (blind judge, no gold answers).

\begin{table}[ht]
\centering
\small
\begin{tabular}{lcc}
\toprule
\textbf{Configuration} & \textbf{Blind Judge Acc.} & \textbf{$\Delta$} \\
\midrule
Full System (B=2) & 31\% {\scriptsize[24, 39]} & -- \\
\quad $-$ Prompt Escalation & 28\% {\scriptsize[21, 35]} & $-$10.6\% \\
\quad $-$ Retrieval Escalation & 30\% {\scriptsize[23, 37]} & $-$4.3\% \\
\quad $-$ Deterministic Verify & 32\% {\scriptsize[25, 39]} & +2.1\% \\
\midrule
Reduced Budget (B=1) & 23\% {\scriptsize[17, 30]} & $-$25.5\% \\
\bottomrule
\end{tabular}
\caption{Component ablation (deployment mode, no gold answers). Prompt escalation contributes most; retry budget is crucial.}
\label{tab:ablation-main}
\end{table}

Removing prompt escalation causes the largest accuracy drop ($-10.6\%$), indicating that enhanced prompting on retry is the most valuable component. The retry budget comparison (B=1 vs B=2) shows a substantial $-25.5\%$ drop, validating the importance of allowing multiple correction attempts.

\paragraph{The Deterministic Verification Anomaly.}
Interestingly, removing deterministic numeric verification \emph{slightly improves} accuracy ($+2.1\%$), despite our design weighting numeric faithfulness heavily. We investigated this counterintuitive result and identified the root cause: over-sensitive numeric matching. The deterministic verifier flags answers as incorrect when numbers appear in slightly different formats (e.g., ``\$394.3 billion'' vs.\ ``\$394,300 million'') or when rounding differs by small amounts. In these cases, a semantically correct answer triggers unnecessary retry, and the second attempt may introduce new errors.

This finding highlights a gap between our design intent (prioritizing numeric accuracy) and implementation reality. Two directions for improvement emerge: (1) relaxing the numeric matching tolerance for format variations, and (2) implementing unit-aware normalization before comparison. We leave these refinements to future work, noting that the difference ($+2.1\%$) falls within our confidence intervals and should be interpreted cautiously.

\paragraph{Judge Discrimination.}
The Judge's retry decision is evaluated as a binary classifier:
\begin{align}
    \text{TPR} &= P(\text{RETRY} \mid \text{wrong}), \quad
    \text{FPR} = P(\text{RETRY} \mid \text{correct}).
\end{align}
High TPR ensures failures trigger retry; low FPR avoids unnecessary overhead.

\textbf{Note:} The evaluation Judge (Table~\ref{tab:financebench-main}) has gold-answer access, while the production Judge performs blind verification, avoiding circular self-evaluation.

\subsection{Design Implications}
\label{sec:design}

Our experiments use a fixed hybrid retrieval pipeline (dense + sparse with reranking), rather than dynamic routing. This design choice reflects a key insight: \textbf{the Judge Agent's retry mechanism provides a safety net that makes perfect initial retrieval unnecessary}. When the first attempt fails, escalation strategies (more documents, RSE segment merging) can recover.

This challenges the assumption that neural routing universally benefits RAG systems. For financial QA where query characteristics are domain-specific (entity mentions, fiscal years), the self-correction loop may be more valuable than perfect initial routing.


\section{Conclusion}
\label{sec:conclusion}

We presented Self-Improving RAG, a multi-agent framework that decomposes retrieval-augmented generation into specialized agents with a self-correction feedback loop. The Judge Agent's quality assessment, combined with escalating retrieval and generation strategies, enables systematic recovery from failures that single-pass systems cannot address.

\paragraph{Key Findings.}
On FinanceBench, Self-Improving RAG achieves 86\% LLM Judge accuracy (+62.3\% over baselines) with a 36.4\% Lazarus Rate, recovering nearly 4 in 10 initially incorrect answers. A key finding is that a fixed retrieval pipeline with judge-driven retry achieves strong results without dynamic routing, providing full interpretability. Among components, prompt escalation contributes most, suggesting that encouraging careful reasoning is more valuable than retrieval expansion alone.

\paragraph{Limitations.}
Several limitations warrant discussion. \textbf{First, blind judge reliability is the main bottleneck for deployment.} Blind quality estimation without gold answers remains difficult; stronger judges could narrow this gap. Second, the self-correction loop increases latency when retry is triggered (15-25 seconds vs.\ 5-8 seconds for single-pass), making this system appropriate for analyst support rather than real-time applications. Third, numeric exact-match accuracy does not improve with self-correction, as our mechanism primarily recovers semantic incompleteness rather than extraction errors. Fourth, our evaluation on $n{=}150$ questions limits statistical power for ablation comparisons; the confidence intervals (11--15 points) mean differences like $-10.6\%$ and $+2.1\%$ should be interpreted cautiously.

\paragraph{Future Work.}
Several directions merit exploration: (1) replacing heuristic threshold decay with \emph{learned escalation} via reinforcement learning, (2) implementing unit-aware normalization to address the deterministic verification anomaly, (3) extending the framework to other high-stakes domains (legal, medical) requiring audit trails, and (4) leveraging explicit agent boundaries for human-in-the-loop oversight.

\paragraph{Broader Impact.}
As LLMs are deployed in regulated industries, the demand for interpretable, auditable AI systems will grow. Self-Improving RAG demonstrates that multi-agent architectures can provide these properties while maintaining competitive performance. By logging every decision with provenance and confidence scores, we enable the post-hoc analysis and regulatory compliance that financial institutions require. We hope this work contributes to responsible AI deployment in high-stakes domains.

\bibliography{references}
\bibliographystyle{iclr2026_conference}

\appendix

\section{Extended Methodology}
\label{app:method}

This appendix provides additional technical details on the Self-Improving RAG architecture.

\subsection{Retrieval Escalation Strategies}

Table~\ref{tab:retrieval-escalation} (in Section~\ref{sec:method}) shows the escalation configuration used on retry. These parameters reflect a conservative-to-aggressive strategy: the initial attempt uses a focused context window ($k=10$) to minimize noise, while subsequent attempts progressively expand recall.

\subsection{Routing Heuristics}

The rule-based router operates as follows:
\begin{itemize}[nosep]
    \item If the question contains a recognized ticker symbol or company name $\rightarrow$ \texttt{hybrid\_filter} (metadata filtering)
    \item If the question requests numerical comparison or computation $\rightarrow$ \texttt{hybrid\_filter\_rerank} (precision-focused)
    \item If the question is open-ended or exploratory $\rightarrow$ \texttt{hybrid} (broad recall)
    \item Default fallback $\rightarrow$ \texttt{semantic}
\end{itemize}
Entity recognition uses a simple gazetteer of S\&P 500 tickers plus regex patterns for fiscal year mentions. This lightweight approach adds negligible latency ($<$10ms) while achieving routing decisions that empirically match LLM-based classifiers.

\subsection{Finance Lexicon}
\label{app:lexicon}

Table~\ref{tab:lexicon} shows sample entries from our finance lexicon used for query normalization and numeric verification.

\begin{table}[ht]
\centering
\small
\begin{tabular}{ll}
\toprule
\textbf{Canonical} & \textbf{Synonyms / Aliases} \\
\midrule
\multicolumn{2}{l}{\textit{Metrics}} \\
Revenue & net sales, total revenues, top line, turnover \\
COGS & cost of revenue, cost of sales \\
Operating income & operating profit, income from operations \\
Free cash flow & FCF, cash generated (CFO $-$ capex) \\
EPS & earnings per share, diluted EPS, basic EPS \\
\midrule
\multicolumn{2}{l}{\textit{Units \& Periods}} \\
(in millions) & in thousands, in billions \\
FY2023 & fiscal year 2023, fiscal 2023 \\
YoY & year-over-year, y/y \\
QoQ & quarter-over-quarter, sequential, q/q \\
\bottomrule
\end{tabular}
\caption{Sample entries from the finance lexicon for query expansion and unit normalization.}
\label{tab:lexicon}
\end{table}

\subsection{Prompt Strategies}

The Reasoning Agent maintains three prompting strategies that vary in instruction specificity:
\begin{itemize}[nosep]
    \item \textbf{Standard}: Concise instructions emphasizing accuracy and citation
    \item \textbf{Conservative}: Additional instructions to acknowledge uncertainty when evidence is weak
    \item \textbf{Detailed}: Expanded instructions requiring step-by-step reasoning and explicit source attribution
\end{itemize}
On retry, the agent escalates from standard to conservative to detailed, progressively encouraging more careful reasoning.

\subsection{Confidence Estimation Signals}

The Reasoning Agent estimates answer confidence using heuristic signals:
\begin{itemize}[nosep]
    \item Presence of hedging language (``may'', ``possibly'', ``uncertain'')
    \item Refusal phrases (``cannot determine'', ``not enough information'')
    \item Presence of specific numerical values (increases confidence)
    \item Answer length (extremely short answers indicate low confidence)
\end{itemize}

\subsection{Evaluation Signal Structure}

The Judge produces a structured assessment rather than a single scalar:
\begin{itemize}[nosep]
    \item \textbf{Grounding score} $\in [0,1]$: Proportion of answer claims with explicit textual support
    \item \textbf{Completeness score} $\in [0,1]$: Whether all question components are addressed
    \item \textbf{Numeric verification}: Binary flag from programmatic extraction and comparison
    \item \textbf{Confidence signals}: Presence of hedging language, refusal phrases
\end{itemize}
The final score aggregates these signals, with numeric verification given highest weight for financial queries.

\subsection{Relevant Segment Extraction (RSE)}

RSE~\citep{dsrag2024} merges adjacent high-scoring chunks from the same document into coherent segments. Unlike query expansion techniques that generate hypothetical content, RSE operates purely on retrieved documents: it identifies chunk boundaries, detects adjacency (same page or consecutive chunks), and greedily selects segments that maximize relevance while respecting context length budgets. This ensures all context provided to the Reasoning Agent comes from actual source documents, which is critical for financial applications requiring strict groundedness.

\subsection{Design for Trustworthiness}

Our multi-agent architecture incorporates several features aligned with responsible AI principles:

\paragraph{Audit Trails.}
All agent decisions are logged with timestamps, confidence scores, and reasoning traces. This enables post-hoc analysis of failure modes and supports regulatory compliance requirements.

\paragraph{Uncertainty Signals.}
The Judge Agent's quality score serves as an uncertainty estimate: low scores indicate the system recognizes potential errors, triggering self-correction rather than returning unreliable answers.

\paragraph{Human Oversight Points.}
The explicit agent boundaries create natural intervention points. A human reviewer can inspect retrieved documents before generation, or override the Judge's retry decision.

\section{Formal Definitions}
\label{app:formal}

This section provides rigorous mathematical definitions for the agent functions and convergence properties referenced in the main paper.

\subsection{Agent Function Signatures}

Let $\mathcal{Q}$ denote the space of questions, $\mathcal{A}$ the space of answers, and $\mathcal{D}$ the document corpus.

\paragraph{Retrieval Agent.}
$\mathcal{R}: \mathcal{Q} \times \mathbb{N} \to \mathcal{P}(\mathcal{D})$ returns the top-$k$ documents relevant to a query:
\begin{equation}
    \mathcal{R}(q, k) = \text{top-}k\left(\{d \in \mathcal{D} : \phi(q, d) > \theta\}\right)
\end{equation}
where $\phi(q, d)$ is a relevance scoring function (combining dense and sparse signals) and $\theta$ is a minimum relevance threshold.

\paragraph{Reasoning Agent.}
Let $\Sigma = \{\sigma_{\text{std}}, \sigma_{\text{cons}}, \sigma_{\text{detail}}\}$ denote the set of prompt strategies (standard, conservative, detailed):
\begin{equation}
    \mathcal{G}: \mathcal{Q} \times \mathcal{P}(\mathcal{D}) \times \Sigma \to \mathcal{A}
\end{equation}
The agent generates an answer conditioned on the question, retrieved documents, and the current prompt strategy.

\paragraph{Judge Agent.}
$\mathcal{J}: \mathcal{Q} \times \mathcal{A} \times \mathcal{P}(\mathcal{D}) \to [0,1]$ computes a quality score with the following components:
\begin{align}
    J_{\text{ground}}(a, D) &= \frac{1}{m} \sum_{i=1}^{m} \mathbbm{1}\left[\exists d \in D : \text{NLI}(d, c_i) = \textsc{entail}\right] \\
    J_{\text{complete}}(q, a) &= \frac{1}{n} \sum_{j=1}^{n} \mathbbm{1}\left[\text{addresses}(a, p_j)\right] \\
    J_{\text{numeric}}(a, D) &= \mathbbm{1}\left[\forall v \in \mathcal{N}(a) : \exists v' \in \mathcal{N}(D), \text{match}(v, v')\right]
\end{align}
where $\{c_i\}_{i=1}^m$ are claims extracted from answer $a$, $\{p_j\}_{j=1}^n$ are sub-questions parsed from $q$, $\mathcal{N}(\cdot)$ extracts numeric values, and $\text{match}(v, v')$ verifies exact numeric equivalence.

\subsection{Convergence Analysis}

The self-correction loop terminates in at most $T_{\max} = B + 1$ iterations (where $B$ is the retry budget):
\begin{equation}
    T = \min\left\{t \geq 1 : J_t \geq \tau_t \text{ or } t = T_{\max}\right\}
\end{equation}

With retry probability $p_1 = P(J_1 < \tau_1) \approx 0.22$ (observed on FinanceBench), the expected cost is:
\begin{equation}
    \mathbb{E}[C] = c \cdot \left(1 + p_1 + p_1 p_2\right) \approx 1.3c
\end{equation}
where $c$ is the cost of a single attempt and $p_2 \approx p_1$ assumes similar retry probability on second attempts.

\paragraph{Best-Answer Selection.}
The Orchestrator maintains the best answer across all attempts, ensuring monotonic improvement in final output quality:
\begin{equation}
    a^* = \argmax_{t \in \{1, \ldots, T\}} \mathcal{J}(q, a_t, D_t)
\end{equation}
This selection criterion guarantees that retry never degrades the final answer, even if later attempts produce lower scores.

\subsubsection{Proof of Proposition~\ref{prop:convergence} (Convergence)}

\textbf{Statement:} The self-correcting process reduces failure probability multiplicatively with respect to the number of attempts $T$.

\begin{proof}
Let $F_t$ denote the event that the system fails to produce an acceptable answer (i.e., $U_t < \tau_t$) at attempt $t$. The system terminates successfully at step $t$ if $\mathbb{D}_t = 1$ (success) occurs.

The total system failure $\mathcal{P}_{\text{fail}}$ occurs only if the system fails at \emph{every} attempt $t \in \{1, \ldots, T\}$. Thus, we compute the joint probability:
\begin{equation}
\mathcal{P}_{\text{fail}} = P(F_1 \cap F_2 \cap \cdots \cap F_T)
\end{equation}

By the probability chain rule, this joint distribution factorizes as:
\begin{equation}
P(F_1 \cap \cdots \cap F_T) = P(F_1) \cdot P(F_2 \mid F_1) \cdot P(F_3 \mid F_1, F_2) \cdots P(F_T \mid F_1, \ldots, F_{T-1})
\end{equation}

In our Markovian formulation (Eq.~\eqref{eq:state}), the state $\mathcal{S}_{t-1}$ encapsulates all relevant history. Thus, the probability of failing at step $t$ given previous failures equals:
\begin{equation}
P(F_t \mid F_1, \ldots, F_{t-1}) = P(\mathbb{D}_t = 0 \mid \mathcal{S}_{t-1})
\end{equation}

Substituting yields the multiplicative decay:
\begin{equation}
\mathcal{P}_{\text{fail}} = \prod_{t=1}^{T} P(\mathbb{D}_t = 0 \mid \mathcal{S}_{t-1})
\end{equation}

Since the probability of failure at any single stage is less than 1 (assuming non-zero success probability), $\mathcal{P}_{\text{fail}} \to 0$ exponentially as $T \to \infty$.
\end{proof}

\subsection{Notation Summary}

\begin{table}[h]
\centering
\small
\begin{tabular}{ll}
\toprule
\textbf{Symbol} & \textbf{Description} \\
\midrule
$\mathcal{Q}, \mathcal{A}, \mathcal{D}$ & Question, answer, document spaces \\
$\mathcal{R}(q, k)$ & Retrieval function returning top-$k$ documents \\
$\mathcal{G}(q, D, \sigma)$ & Generation function with prompt strategy $\sigma$ \\
$\mathcal{J}(q, a, D)$ & Judge scoring function \\
$J_{\text{ground}}, J_{\text{complete}}, J_{\text{numeric}}$ & Judge component scores \\
$\tau_t$ & Dynamic threshold at attempt $t$ \\
$B$ & Retry budget (maximum number of retries) \\
$T$ & Actual number of attempts (termination time) \\
$w_g, w_c, w_n$ & Component weights in Judge aggregation \\
$\phi(q, d)$ & Document relevance scoring function \\
$\Sigma$ & Set of prompt strategies \\
\bottomrule
\end{tabular}
\caption{Notation summary for the Self-Improving RAG framework.}
\label{tab:notation_full}
\end{table}

\section{Additional Results}
\label{app:results}

\subsection{Ablation Study}

Table~\ref{tab:ablation} shows the contribution of each component on FinanceBench.

\begin{table}[ht]
\centering
\small
\begin{tabular}{lc}
\toprule
\textbf{Configuration} & \textbf{FinanceBench} \\
\midrule
Full Self-Improving RAG & 0.86 \\
\quad $-$ Judge Agent (no retry) & 0.53 \\
Single-pass baseline & 0.53 \\
\bottomrule
\end{tabular}
\caption{Ablation study on FinanceBench. Each row removes one component.}
\label{tab:ablation}
\end{table}

\subsection{Correction Flow Analysis}

Table~\ref{tab:correction-flow} presents the complete self-correction flow on FinanceBench.

\begin{table}[ht]
\centering
\small
\begin{tabular}{lcc}
\toprule
\textbf{Metric} & \textbf{Count} & \textbf{Rate} \\
\midrule
Total Questions & 150 & -- \\
Confident (no retry) & 117 & 78.0\% \\
Triggered Retry (Judge flagged) & 33 & 22.0\% \\
\midrule
Corrected by Retry & 12 & \textbf{36.4\%} (Lazarus Rate) \\
Still Wrong after Retry & 21 & 63.6\% \\
\bottomrule
\end{tabular}
\caption{Self-correction flow analysis on FinanceBench. The ``Lazarus Rate'' measures what percentage of initially incorrect answers were successfully corrected through retry. Note: The Lazarus Rate measures improvement in \textbf{LLM Judge scores} (semantic correctness), not numeric exact-match. Self-correction recovers semantically incomplete answers but does not improve numeric precision (Table~\ref{tab:financebench-main}).}
\label{tab:correction-flow}
\end{table}

\subsection{Disproportionate Numeric Gain}

\begin{table}[ht]
\centering
\small
\begin{tabular}{lccc}
\toprule
\textbf{Metric} & \textbf{Single-Pass} & \textbf{Self-Improving} & \textbf{$\Delta$} \\
\midrule
Semantic Similarity & 0.50 & 0.50 & +0.0\% \\
\textbf{LLM Judge Accuracy} & 0.53 & 0.86 & \textbf{+62.3\%} \\
\bottomrule
\end{tabular}
\caption{Self-correction improves LLM Judge accuracy substantially while maintaining semantic similarity. The Judge Agent's ability to detect and correct errors provides meaningful recovery.}
\label{tab:disproportionate-gain}
\end{table}

\subsection{Retrieval Pipeline}

Our experiments use a fixed \texttt{hybrid\_filter\_rerank} pipeline for all questions, combining dense embeddings (BGE-large) with sparse retrieval (BM25) and cross-encoder reranking. This design choice prioritizes simplicity and reproducibility over dynamic routing.

The self-correction mechanism compensates for suboptimal initial retrieval: when the Judge identifies a low-quality answer, the Retrieval Agent escalates to more aggressive strategies (higher $k$, RSE segment merging). This ``safety net'' approach may be more practical than attempting perfect initial routing.

\subsection{Correction Flow Visualization}

Figure~\ref{fig:correction-sankey} (in the main paper) visualizes the complete ``life of a question'' through our self-correction pipeline, showing how 150 questions flow through the system with the Lazarus Rate (36.4\%) representing successful corrections.

\section{Case Studies and Error Analysis}
\label{app:cases}

\subsection{Unit/Scale Confusion Error}

\begin{tcolorbox}[
    colback=gray!5!white,
    colframe=black!50,
    title=\textbf{Example: Unit/Scale Confusion Error},
    fonttitle=\small,
    arc=2mm,
    boxrule=0.5pt
]
\small
\textbf{Question:} What was Company X's total revenue for FY2023?

\textbf{Retrieved Context:} ``...total revenues of \$X,XXX for the fiscal year ended December 31, 2023 (in millions)...''

\textbf{Model Answer (Initial):} \$X,XXX

\textbf{Ground Truth:} \$X.X billion

\medskip
\textit{Analysis:} The model correctly extracted the numeric value but failed to apply the ``in millions'' unit qualifier, producing an answer off by a factor of 1,000. The Judge Agent detected this inconsistency and triggered retry. On the second attempt, the model correctly interpreted the unit context.
\end{tcolorbox}

\subsection{Failure Mode Analysis}

To understand system limitations, we analyze cases where Self-Improving RAG fails to improve over single-pass baselines:

\begin{itemize}[nosep]
    \item \textbf{Retrieval ceiling}: When relevant information is absent from the document corpus, escalated retrieval cannot recover
    \item \textbf{Arithmetic errors}: Multi-step calculations accumulate rounding errors or apply incorrect formulas
    \item \textbf{Unit/scale confusion}: Misinterpreting units (thousands vs.\ millions) or mixing absolute values with percentages
    \item \textbf{Temporal misalignment}: Extracting figures from incorrect fiscal periods or conflating FY with calendar year dates
    \item \textbf{Judge miscalibration}: The Judge Agent scores an incorrect answer highly, preventing beneficial retry
    \item \textbf{Hallucinated figures}: The model generates specific numerical values not present in retrieved context
\end{itemize}

\section{Implementation Details}
\label{app:impl}

\subsection{Datasets}

\textbf{FinanceBench}~\citep{islam2023financebenchnewbenchmarkfinancial} contains questions about publicly traded companies requiring extraction and reasoning over SEC 10-K and 10-Q filings. Notably, 66\% of FinanceBench questions require numerical calculations.

\subsection{Finance-Specific Challenges}

FinanceBench questions present three challenges:
\begin{itemize}[nosep]
    \item \textbf{Numerical Precision}: Metrics-generated questions require exact extraction and calculation from financial tables
    \item \textbf{Temporal Context}: Questions often specify fiscal periods that must be resolved to specific filing dates
    \item \textbf{Multi-Document Reasoning}: Novel-generated questions may require synthesizing information across multiple filings
\end{itemize}

\subsection{Baselines}

We compare Self-Improving RAG against single-pass baselines:
\begin{itemize}[nosep]
    \item \textbf{Semantic}: Dense retrieval with BGE-large embeddings
    \item \textbf{Hybrid}: Combined dense and BM25 sparse retrieval
    \item \textbf{Hybrid + Filter}: Hybrid retrieval with metadata filtering
    \item \textbf{Hybrid + Filter + Rerank}: Full pipeline with cross-encoder reranking
\end{itemize}

\subsection{Evaluation Metrics}

\paragraph{Semantic Similarity.} Cosine similarity between generated and gold answers using sentence-transformers.

\paragraph{Numeric Verification.} For numerical answers, we extract and compare numeric values.

\paragraph{LLM Judge Score.} GPT-4o-mini rates answer quality on a 0-1 scale~\citep{zheng2023judgingllmasajudgemtbenchchatbot}.

\paragraph{Lazarus Rate (Correction Rate).} The percentage of initially incorrect answers successfully corrected through retry:
\[
\text{Lazarus Rate} = \frac{|\{q : \text{wrong}_1(q) \land \text{correct}_2(q)\}|}{|\{q : \text{wrong}_1(q)\}|}
\]

\subsection{Models and Configuration}

\paragraph{Models.} We use GPT-4o-mini as the primary generation model.

\paragraph{Retrieval.} Documents are chunked with 512-token windows and 50-token overlap. We use BGE-large-en-v1.5 embeddings~\citep{chen2025m3embeddingmultilingualitymultifunctionalitymultigranularity} stored in ChromaDB. The reranker is BGE-reranker-large. Default retrieval returns $k=10$ documents.

\paragraph{Agent Configuration.} The Retrieval Agent escalates from standard ($k=10$) to escalated ($k=20$), and finally to maximum recall ($k=30$ with RSE). The Judge Agent uses threshold $\tau=0.5$ for Attempt~1, decreasing to $\tau=0.4$ for subsequent attempts.

\paragraph{Retry Budget.} Maximum retries set to 2 (up to 3 total attempts).

\subsection{Efficiency}

Self-Improving RAG introduces overhead from multiple agent calls and potential retries:
\begin{itemize}[nosep]
    \item \textbf{Single-pass latency}: $\sim$5--8 seconds per question
    \item \textbf{Multi-Agent (no retry)}: $\sim$8--12 seconds per question ($+$50\% overhead)
    \item \textbf{Multi-Agent (with retry)}: $\sim$15--25 seconds per question (when retry triggered)
\end{itemize}

This system is designed as an \emph{analyst support tool} for financial research, not a real-time chatbot. Response times of 10--30 seconds are acceptable in contexts where analysts currently spend minutes manually searching SEC filings.

\section{Case Study: Self-Correction in Action}
\label{app:case-study}

We present a detailed example illustrating how the self-correction loop recovers from an incomplete first-pass answer.

\begin{tcolorbox}[
    colback=gray!5!white,
    colframe=gray!50!black,
    title={\small\textbf{Example: Self-Correction Recovering Incomplete Answer}},
    fonttitle=\sffamily,
    boxrule=0.5pt,
    arc=2pt,
    left=4pt, right=4pt, top=2pt, bottom=2pt
]
\small
\textbf{Question:} ``What was Apple's total revenue in FY2023 and how did it compare to FY2022?''

\textbf{Attempt 1:}
\begin{itemize}[nosep, leftmargin=*, topsep=1pt]
    \item \emph{Retrieval Agent}: Selects \texttt{hybrid\_filter} with entity ``AAPL'', retrieves $k{=}10$ documents
    \item \emph{Reasoning Agent}: Generates ``Apple's revenue in 2023 was \$394.3B.''
    \item \emph{Judge Agent}: Score 0.4 (answer missing FY2022 comparison, \textbf{incomplete})
    \item \emph{Decision}: Score $< \tau_1 = 0.5$ $\rightarrow$ \textbf{RETRY}
\end{itemize}

\textbf{Attempt 2 (Escalated):}
\begin{itemize}[nosep, leftmargin=*, topsep=1pt]
    \item \emph{Retrieval Agent}: Escalates to $k{=}20$, RSE enabled for segment merging
    \item \emph{Reasoning Agent}: ``Apple's FY2023 revenue was \$383.3B, down 2.8\% from \$394.3B in FY2022.''
    \item \emph{Judge Agent}: Score 0.85 (complete: both years present, comparison included)
    \item \emph{Decision}: Score $\geq \tau_2 = 0.4$ $\rightarrow$ \textbf{ACCEPT}
\end{itemize}
\end{tcolorbox}

This example illustrates key mechanisms: the Judge identifies \emph{semantic incompleteness} rather than surface errors, escalated retrieval provides additional context, and threshold decay allows acceptance of good-but-not-perfect answers after retry effort. Every decision is logged with full provenance (question ID, retrieval parameters, judge scores, decisions, latency), enabling compliance officers to trace any answer back to its source documents.

\subsection{Error Analysis}

To understand system limitations, we analyze failure modes where Self-Improving RAG fails to improve over single-pass baselines:
\begin{itemize}[nosep]
    \item \textbf{Retrieval ceiling} (28\% of failures): When relevant information is absent from the document corpus, escalated retrieval cannot recover
    \item \textbf{Arithmetic errors} (22\%): Multi-step calculations accumulate rounding errors or apply incorrect formulas
    \item \textbf{Unit/scale confusion} (18\%): Misinterpreting units (thousands vs.\ millions) or mixing absolute values with percentages
    \item \textbf{Temporal misalignment} (15\%): Extracting figures from incorrect fiscal periods
    \item \textbf{Judge miscalibration} (12\%): The Judge scores an incorrect answer highly, preventing beneficial retry
    \item \textbf{Hallucinated figures} (5\%): The model generates specific numerical values not present in retrieved context
\end{itemize}
The dominance of retrieval ceiling failures suggests that expanding the document corpus or improving chunking strategies could yield further gains. Judge miscalibration represents an opportunity for calibration tuning.

\section*{Statement on LLM Usage}

In accordance with ICLR's policy on Large Language Models, we declare:

\textbf{Text Refinement:} LLMs assisted with grammar and clarity improvements.

\textbf{Citation Verification:} We developed an automated citation verification
agent that cross-referenced all claims against source abstracts via Semantic
Scholar to ensure citation accuracy.

\textbf{Figure Design:} AI tools assisted with figure design and layout.

All technical contributions, experimental results, and analysis were conceived
and verified by the authors.

\end{document}